\documentclass[final,12pt]{clear2023} 
\usepackage{tikz}
\usepackage{cleveref}
\usepackage{xcolor}
\usepackage{adjustbox}
\usepackage{pgfplots}
\usepackage{xspace}
\usepackage{comment}

\DeclareMathOperator*{\argmin}{arg\,min}

\usepackage[acronym]{glossaries}
\glsdisablehyper
\newacronym{ml}{ML}{Machine Learning}
\newacronym{ica}{ICA}{Independent Component Analysis}
\newacronym{ood}{OOD}{Out of Distribution Generalization}
\newacronym{dag}{DAG}{Directed Acyclic Graph}
\newacronym{scm}{SCM}{Structural Causal Model}
\newacronym{dscm}{DSCM}{Deep Structural Causal Model}
\newacronym{gan}{GAN}{Generative Adversarial Network}
\newacronym{vae}{VAE}{Variational Auto-Encoder}
\newacronym{nf}{NF}{Normalizing Flow}
\newacronym{elbo}{ELBO}{Evidence Lower Bound}
\newacronym{mse}{MSE}{Mean Squared Error}
\newacronym{pdf}{PDF}{Probability Distribution Function}
\newacronym{cdf}{CDF}{Cumulative Distribution Function}

\title[Counterfactual (Non-)identifiability of Learned Structural Causal Models]{Counterfactual (Non-)identifiability of Learned SCMs}
\usepackage{times}

\clearauthor{%
 \Name{Arash Nasr-Esfahany} \Email{arashne@mit.edu}\\
 \addr Massachusetts Institute of Technology%
 \AND
 \Name{Emre K\i c\i man} \Email{emrek@microsoft.com}\\
 \addr Microsoft Research%
}

\begin{document}

\maketitle

\begin{abstract}%
Recent advances in probabilistic generative modeling have motivated learning Structural Causal Models (SCM) from observational datasets using deep conditional generative models, also known as \emph{\gls{dscm}}. 
If successful, \gls{dscm}s can be utilized for causal estimation tasks, e.g., for answering counterfactual queries~(\cite{Pawlowski2020DeepSC}).
In this work, we warn practitioners about non-identifiability of counterfactual inference from observational data, even in the absence of unobserved confounding and assuming known causal structure.
We prove counterfactual identifiability of monotonic generation mechanisms with single dimensional exogenous variables.
For general generation mechanisms with multi-dimensional exogenous variables, we provide an impossibility result for counterfactual identifiability, motivating the need for parametric assumptions.
As a practical approach, we propose a method for estimating worst-case errors of learned \gls{dscm}s' counterfactual predictions.
The size of this error can be an essential metric for deciding whether or not \gls{dscm}s are a viable approach for counterfactual inference in a specific problem setting.
In evaluation, our method confirms negligible counterfactual errors for an identifiable \gls{scm} from prior work, and also provides informative error bounds on counterfactual errors for a non-identifiable synthetic \gls{scm}.
\end{abstract}

\begin{keywords}%
  Counterfactual Inference, Identifiability, Conditional Generative Models, Deep Structural Causal Models
\end{keywords}

\section{Introduction}

\newcommand{\bpara}[2]{#2}
\newcommand{\si}{}

\newcommand{\emk}[2]{{\textcolor{blue}{{\bf #1:} #2 --emre}}}

Counterfactuals queries about hypothetical worlds reside in the third rung of the ladder of causation~(\cite{why}).
They investigate alternative imaginary outcomes had some precondition been different while everything else stayed the same.
Such retrospective ``what-if'' questions are critical for policy analyses~(\cite{si}), root-cause analyses~(\cite{root-cause}), credit assignment~(\cite{credit}), harm measurement~(\cite{harm}), legal analyses~(\cite{legal}), fairness~(\cite{fairness}), explainability~(\cite{explainability}), and trace-driven simulation~(\cite{causalsim}), across a variety of domains from epidemiology and medicine~(\cite{medicine}) to economics, computer systems~(\cite{sage}), and advertising~(\cite{ads}).
For example, in a medical scenario, we might ask whether a specific patient would have healed if they had not taken a drug; in policy analyses, we might ask whether societal measures would have been higher if government had enacted different economic policies; and in business scenarios, we might ask whether customers would have bought more products if a company had used a different marketing strategy.

The main challenge in counterfactual inference from \emph{observational data} is that both the generation mechanisms and latent world properties are unknown.
If we knew the generation mechanisms, counterfactual inference would have been straightforward~(\cite{causality-book}).
Motivated by the developments in probabilistic generative modeling, a recent line of work~(\cite{Dash2022EvaluatingAM,Khemakhem2021CausalAF,Pawlowski2020DeepSC,Sanchez2022DiffusionCM}) suggests learning structured generative models (\gls{dscm}s) by fitting the observed data distributions.
This naturally raises the following question: Given the fact that identification is an essential part of every causal estimation~(\cite{dowhy,pch}), to what extent can we trust counterfactual estimates generated by \gls{dscm}s?

In this work, we investigate counterfactual identifiability of \gls{dscm}s from observational data.
We prove counterfactual identifiability for monotonic generation mechanism with single dimensional exogenous variables~(\Cref{sub:identifiable}).
For generation mechanisms with multi-dimensional exogenous variables, we prove an impossibility result by proposing a general counter-example for counterfactual identifiability~(\Cref{sub:unidentifiability}) which underscores the need for parametric assumptions.
Going through the counterfactual identification process for each set of assumptions is a cumbersome process.
Also, precise counterfactual identifiability may be too strong.
Sufficiently small errors in counterfactual estimates could be enough for practitioners.
Hence, in \Cref{sec:method}, we propose a computational method to measure the counterfactual ambiguity of learned generation mechanisms.

We evaluate our method with an identifiable \gls{scm} adopted from prior work, and also with a non-identifiable synthetic \gls{scm}.
In the identifiable case, our method confirms counterfactual identifiability by showing negligible error bounds for counterfactual estimates.
In the non-identifiable case, our method provides informative error bounds for counterfactual estimates, which can be used by practitioners as an important metric for deciding whether or not using \gls{dscm}s is a viable option in their specific problem of interest with their specific set of assumptions.
\section{Background}
\label{sec:background}

\noindent \textbf{Notation:} We use capital letters for referring to random variables (e.g., $V$),
lowercase letters for referring to realizations of random variables (e.g., $v$),
boldface font for referring to a set of random variables (e.g., $\boldsymbol{V} = \{V_1, \ldots, V_n\}$),
and $\,{\buildrel d \over =}\,$ for referring to distributional equivalence.

\subsection{Structural Causal Model (SCM)}
An \gls{scm} is our causal model of the world.
It explains how nature assigns values to a set of variables of interest.
It allows us to deal rigorously with causal queries by pinning down the causal story behind the generation of our data.

Formally, an \gls{scm} $\mathcal{M}$ consists of two sets of variables $\boldsymbol{U} = \{U_1, \ldots, U_n\}$ and $\boldsymbol{V} = \{V_1, \ldots, V_n\}$.
Variables in $\boldsymbol{U}$ are called \emph{exogenous variables} as they are external to $\mathcal{M}$, and we do not model their generation.
Variables in $\boldsymbol{V}$ are called \emph{endogenous variables} as they are internal to $\mathcal{M}$, and we do model their generation.
$\mathcal{M}$ also includes a set of functions $\boldsymbol{f} = \{f_1, \ldots, f_n\}$ that describe generation of each endogenous variable based on the values of other variables in the model, i.e., $\forall i: X_i := f_i(U_i, \boldsymbol{PA}_i)$ where $\boldsymbol{PA}_i \subseteq \boldsymbol{V} \backslash V_i$ is the subset of $\boldsymbol{V}$ that are parents of $V_i$ in the causal \gls{dag}.
Throughout this paper, we assume absence of any unobserved confounders, or a Markovian \gls{scm}.
This means that exogenous variables are statistically independent and their distribution factorizes, i.e., $\forall \boldsymbol{u}:p_U(\boldsymbol{u}) = \prod_{i=1}^n p_{U_i}(u_i)$.
The collection of generation functions $\boldsymbol{f}$ and the prior distribution over $\boldsymbol{U}$ entails a joint distribution over $p^{\mathcal{M}}(\boldsymbol{V}, \boldsymbol{U})$.

\subsection{Interventions and the Do-operator}
In \gls{scm} $\mathcal{M}$, let $\boldsymbol{X}$ be a subset of endogenous variables, and $\boldsymbol{x}$ be their realizations.
The submodel $\mathcal{M}_{\boldsymbol{x}}$ is the \gls{scm} with the same exogenous and endogenous variables, and prior distribution over $\boldsymbol{U}$ as $\mathcal{M}$.
However, its set of generation functions change to
\begin{equation*}
\boldsymbol{f}_{\boldsymbol{x}} = \{f_i | V_i \notin \boldsymbol{X}\} \cup \{f'_j(U_j, \boldsymbol{PA}_j) := x_j | V_j \in \boldsymbol{X}\}.    
\end{equation*}
The do-operation $\text{do}(\boldsymbol{X} := \boldsymbol{x'})$ represents applying the \emph{atomic intervention} that changes the \gls{scm} from $\mathcal{M}$ to $\mathcal{M}_{\boldsymbol{x'}}$.
Atomic interventions and the corresponding do-operations can be generalized to other interventions (e.g., functional, soft, etc.)~(\cite{causality-book,soft-calculus,soft2,soft3}).

\subsection{Counterfactuals}
\label{sub:cf}
Unlike interventional queries that operate at the population level, providing aggregate statistics about the effect of interventions (i.e., exogenous variables $\boldsymbol{U}$ are sampled from the prior exogenous distribution $p_{\boldsymbol{U}}(\cdot)$), counterfactual queries operate at the unit level.
Given the evidence $\boldsymbol{x}$, where $\boldsymbol{x}$ is a realization of a subset of endogenous variables $\boldsymbol{V}$, \emph{what would $\textbf{V}$ have been had we done the intervention $\text{do}(\boldsymbol{X} := \boldsymbol{x'})$}? This is an example of a counterfactual query which we formally represent as $p^{\mathcal{M}_{\boldsymbol{x'}}}_{\boldsymbol{V}}(\cdot|\boldsymbol{x})$ or in short, $p_{\boldsymbol{V}_{\boldsymbol{x'}}}(\cdot|\boldsymbol{x})$.

\cite{causality-book} provides an algorithmic procedure for answering counterfactual queries, given the underlying \gls{scm} $\mathcal{M}$:
\begin{enumerate}
    \item \textbf{Abduction:} predict the distribution of exogenous variables given the observed evidence, i.e., infer the posterior $p_{\boldsymbol{U}}(\cdot|\boldsymbol{x})$.
    \item \textbf{Action:} Perform the desired intervention $\text{do}(\boldsymbol{X} := \boldsymbol{x'})$, and also change the distribution of exogenous variables to the inferred posterior in the abduction step ($p_{\boldsymbol{U}}(\cdot|\boldsymbol{x})$).
    In other words, change change the \gls{scm} $\mathcal{M}$ to $\mathcal{M}_{\boldsymbol{x'}|\boldsymbol{x}}$.
    \item \textbf{Prediction:} Estimate the counterfactual distribution of interest $p_{\boldsymbol{V}_{\boldsymbol{x'}}}(\cdot|\boldsymbol{x})$ using the modified \gls{scm} $\mathcal{M}_{\boldsymbol{x'}|\boldsymbol{x}}$ and its entailed distribution, calculate the counterfactual query of interest.
\end{enumerate}

For a more extensive overview of the causal literature, see \cite{causality-book,elements,primer}

\subsection{Deep Structural Causal Model (DSCM)}
\label{sub:dscm}
The three-step procedure described in \Cref{sub:cf} assumes a known \gls{scm}, including known generation functions $\boldsymbol{f}$.
However, this knowledge is commonly unavailable in real-world problems.
To overcome this challenge, \cite{Pawlowski2020DeepSC} propose a \gls{dscm} methodology for learning the generation functions (and also their corresponding inference mechanisms) from the observed data $p_{\boldsymbol{V}}(\cdot)$.
We refer to the set of learned generation mechanisms as $\hat{\boldsymbol{f}} = \{\hat{f_1}, \ldots, \hat{f_n}\}$.

\gls{dscm}s use deep conditional generative models for modeling the generation mechanism of each of the endogenous variables, i.e., $f_i, i \in \{1, \ldots, n\}$.
Deep conditional generative models start from samples of a tractable probability distribution, e.g., standard Gaussian, and warp this tractable distribution (given the conditions) to resemble the empirical conditional distribution observed in the data.
They suggest three different categories of generative modeling for learning the generation functions: ``\emph{invertible, explicit}'' which uses Normalizing Flows (NF)~(\cite{NF1, NF2}), ``\emph{amortized, explicit}'' which uses Variational Auto-encoders (VAE)~(\cite{vae}), and ``\emph{amortized, implicit}~(\cite{implicit})'' that uses Generative Adversarial Networks (GAN)~(\cite{gan,wgan,ali}).
They argue that learning generation functions and their corresponding inference mechanisms (a tractable procedure for estimating $p_{U_i|\boldsymbol{PA}_i, V_i}(\cdot|\boldsymbol{pa_i}, y_i)$ which is necessary for the abduction step) enables tractable counterfactual inference from observational data, without any analysis of counterfactual identifiability.
For more details, see \cite{Pawlowski2020DeepSC}.

As learning the generation function of $V_i$ does not have any effect on learning the generation function of $V_j$ for $i \neq j$, we focus on learning the generation function of a single node $V_i$ henceforth.
For ease of notation, we drop the subscript $i$, refer to the node for which we are learning the generation mechanism as $Y$, its causal parents as $\boldsymbol{T}$, its corresponding exogenous variables as $U$, its generation mechanism given its parents as $f(\cdot, \boldsymbol{T})$, and the learned conditional generation mechanism parameterized by $\theta$ as $\hat{f_\theta}(\cdot, \boldsymbol{T})$.
\section{Related Work}
\label{sec:related}

Recently, there has been a surge of interest in using deep learning for answering counterfactual queries.
\cite{Pawlowski2020DeepSC} develop a framework for learning \gls{dscm}s from observational datasets, and using them for answering both interventional and counterfactual queries.
\cite{deci} uses \gls{nf}s~(\cite{spline-flows,flow2,flow3,flow4}) and variational inference~(\cite{vi1,vi2,vi3}) for causal discovery~(\cite{discovery1}) and learning the non-linear additive noise \gls{scm}~(\cite{anm1,anm2,anm3}) that can further be utilised for answering causal queries including counterfactuals.
The idea of using \gls{nf}s for leaning the \gls{scm} and counterfactual inference was also proposed in \cite{Khemakhem2021CausalAF}.
Counterfactual image generation has also been explored using \gls{gan}s~(\cite{ali}) and Denoising Diffusion Models~(\cite{score,denoising}) in \cite{Dash2022EvaluatingAM} and \cite{Sanchez2022DiffusionCM}, respectively.
\cite{causalsim} learns generation mechanisms from interventional data, and does counterfactual inference for trace-driven simulation assuming low-rank generation mechanisms.
\cite{rlcf} uses conditional \gls{gan}s~(\cite{bidirectional-gan}) to learn the step-wise generation mechanism for generating counterfactual experience for sample efficient reinforcement learning~(\cite{rl}), assuming monotonicity of generation mechanisms.

\cite{shpitser2008complete,cf-ident1,cf-ident2} perform non-parametric counterfactual identification using combination of observational and experimental data.
They do not take into account parametric and functional form assumptions, as we do in this work.
Also, our definition of counterfactual identifiability (\Cref{def:gen}) differs as it implies identifiability of every counterfactual query.
\cite{disc1,disc2,disc3} analyze counterfactual identifiability of discrete \gls{scm}s.
\cite{partial} provides bounds on counterfactual estimates for \gls{scm}s with discrete endogenous variables by mapping it to polynomial programming~(\cite{polynomial1,polynomial2}) and estimating the answer with Monte Carlo algorithms.
\cite{emp1,emp2} propose computational methods for identification of interventional queries, while our proposed method~(\Cref{sec:method}) estimates counterfactual ambiguity of \gls{dscm}s.
\section{Motivation}
\label{sec:motiv}
In this section, we use a simple example to demonstrate how \gls{dscm}s can fail in counterfactual inference. Consider the following simple generation mechanism for $Y$ given its single dimensional binary parent $T$:
\begin{equation}
    T \sim Bernoulli(0.5),\; \; U \sim Unif(0, 1), \; Y^1 = 
    \begin{cases} 
    U, &T = 1\\
    U-1, &T = 0          
    \end{cases}
\label{eqn:simple_scm}
\end{equation}
Suppose we have access to a sufficiently large dataset of observed variables $(t, y)$ generated according to \Cref{eqn:simple_scm}.
Furthermore, suppose we are interested in answering a counterfactual query, for instance, \emph{What would have happened to $Y$ had we done $t'=1$, given that $Y$ was $y$ when $T$ was $t=0$ (evidence)?}

As described in \Cref{sub:dscm}, prior work suggests using conditional generative models for learning the conditional distribution of $P_{Y|T}(\cdot|t)$ for different values of $t$. 
The learned conditional distribution parameterized by a neural network can then be interpreted as the data generation mechanism, or \gls{dscm}, and further utilized for counterfactual inference.

In our simple, motivating example, many generation mechanisms exist that match the observational conditional distribution $p_{Y|T}(\cdot|t)$.
For instance, consider the following alternative data generation mechanism:

\begin{equation}
    T \sim Bernoulli(0.5),\; \; U \sim Unif(0, 1), \; Y^2 =
    \begin{cases} 
    U, &T = 1\\
    -U, &T = 0          
    \end{cases}
\label{eqn:simple_scm2}
\end{equation}

\noindent Note that data generated according to \Cref{eqn:simple_scm} and \Cref{eqn:simple_scm2} have the same distribution, i.e.,
\begin{equation*}
    \forall y, t: p_{Y^1|T}(y|t) = p_{Y^2|T}(y|t)
\end{equation*}
This means that both \gls{scm}s are potentially valid \gls{dscm}s, as fitting the observational distribution is the objective function for learning \gls{dscm}s.

However, they give different answers to counterfactual queries. 
Recall the counterfactual query of interest in this example:  What would have happened to $Y$ had we done $t'=1$, given that $Y$ was $y$ when $T$ was $t=0$ (evidence)?
We use \Cref{eqn:simple_scm,eqn:simple_scm2} to answer this query:
\begin{equation*}
    p_{Y^1_{t'}}(\cdot|t, y) = \delta(1+y), \; 
    p_{Y^1_{t'}}(\cdot|t, y) = \delta(-y)
\end{equation*}

By demonstration, this example underscores the risks of na\"ively using \gls{dscm}s for counterfactual inference without an identification analysis.
\section{Counterfactual Identifiablity}
\label{sec:identifiability}
As demonstrated by the motivational example in \Cref{sec:motiv}, using \gls{dscm}s without any identification analysis could result in inaccurate counterfactual estimates.
In this section, we formally define counterfactual identifiability with \gls{dscm}s.
Additionally, we prove counterfactual identifiability for bijective \gls{scm}s that satisfy a set of assumptions.
Finally, for \gls{scm}s with multi-dimensional exogenous variables, we provide an impossibility result by proposing a general method for creating counterexamples, i.e., creating alternative \gls{scm}s that fit the observational data distribution equally well, while giving different answers to counterfactual queries.

Similar to \cite{Pawlowski2020DeepSC}, we assume knowledge of the exact causal structure, i.e., causal \gls{dag}, and the absence of any unobserved confounding.
We focus on learning a single conditional generative model (and a corresponding inference mechanism) for the generation of the endogenous node $Y$, given its parents $\boldsymbol{T}$ in the causal graph $(\boldsymbol{T} \rightarrow Y)$ since conditional generative models learned for generation of different endogenous nodes do not affect training of each other.

As explained in \Cref{sub:dscm}, the learned conditional generative model for this specific node of the causal \gls{dag} maps exogenous variable $U$ sampled from a tractable distribution to samples of the observed conditional distribution $p_{Y|\textbf{T}}(\cdot|\textbf{T})$, conditioned on the value of its causal parents $\textbf{t}$.

\begin{definition}
\label{def:gen}
(Counterfactual identifiability with \gls{dscm}s):
Using \gls{dscm}s for learning generation mechanisms leads to identifiable counterfactual estimates if all \gls{dscm} pairs that fit the observational conditional distribution result in equivalent counterfactual estimates for any counterfactual query of interest.
In other words,
\begin{align*}
    &\forall \boldsymbol{t}: f_{\theta_1}(\cdot, \boldsymbol{t}) \,{\buildrel d \over =}\, p_{Y|\boldsymbol{T}}(\cdot|\boldsymbol{t})
    , f_{\theta_2}(\cdot, \boldsymbol{t}) \,{\buildrel d \over =}\, p_{Y|\boldsymbol{T}}(\cdot|\boldsymbol{t})\\
    \rightarrow
    &\forall \boldsymbol{t}, \boldsymbol{t'}: p_{Y_{\boldsymbol{t'}}}(\cdot|\boldsymbol{t}, y;\theta_1) \,{\buildrel d \over =}\, p_{Y_{\boldsymbol{t'}}}(\cdot|\boldsymbol{t}, y;\theta_2)
\end{align*}
\end{definition}
Intuitively, this definition means that fitting the observational conditional distribution is a good proxy for estimating counterfactuals if all \gls{dscm}s that fit the observational conditional distribution lead to the same set of counterfactual estimates.

\subsection{Bijective Generation Mechanisms}
\label{sub:bijective}
\begin{definition}
(Bijective Generation Mechanism):
Generation mechanism of $Y$ from its parents $\boldsymbol{T}$ and the corresponding exogenous variable $U$ is bijective if the mapping from $U$ to $Y$ is a bijection for each realization of parents.
\end{definition}
In bijective generation mechanisms, both the generation and inference mechanisms are deterministic functions, and no information is lost in generation.
In this family, we have an equivalent definition of counterfactual identifiability with \gls{dscm}s:
\begin{definition}
\label{def:bijection}
(Counterfactual identifiability of bijective \gls{scm}s):
Using \gls{dscm}s for learning bijective generation mechanisms leads to identifiable counterfactual estimates if all \gls{dscm} pairs that fit the observational conditional distribution, recover the same exogenous variable up to a nonlinear invertible indeterminacy $g(\cdot)$.
Note that $g(\cdot)$ does not depend on parents $\boldsymbol{T}$.
\end{definition}
\begin{theorem}
Definition \ref{def:bijection} and Definition \ref{def:gen} are equivalent for bijective generation mechanisms.
\end{theorem}
\begin{proof}
Let the ground truth data generation mechanism of $Y$ based on its parents $\boldsymbol{T}$ be $f$, i.e., $Y = f(U, \boldsymbol{T})$.
First, we show how Definition \ref{def:bijection} implies Definition \ref{def:gen} for bijective mechanisms.
If the learned \gls{dscm} recovers the latent variable upto some nonlinear invertible indeterminacy ($\hat{U} = g(U)$), it means that the learned generation mechanism $\hat{f}$ is equal to the combination of the ground truth generation mechanism $f$ and the inverse of $g$, i.e., $\forall \hat{u}, \boldsymbol{t}: \hat{f}(\hat{u}, \boldsymbol{t}) = f(g^{-1}(\hat{u}), \boldsymbol{t})$.
This means that in the abduction step, the learned model recovers $g(U)$, but in the prediction step with the modified submodel, it first applies the inverse of $g$ to the recovered exogenous variable, and then the ground truth generation mechanism $f$.
Thus, all counterfactual predictions are equal to those of the ground truth generation mechanism $f$, which implies Definition \ref{def:gen}.

Next, we prove that Definition \ref{def:gen} implies Definition \ref{def:bijection} by contradiction.
Without loss of generality, suppose that parent is binary and can only take values $\boldsymbol{t_1}$ and $\boldsymbol{t_2}$.
Also, assume that the learned model recovers $\hat{U}$, where $\hat{U} = g_1(U)$ when the evidence is $\boldsymbol{t_1}$ and $\hat{U} = g_2(U)$ when evidence is $\boldsymbol{t_2}$.
This means that the learned generation mechanism is 
\begin{equation*}
\hat{f}(\hat{U}, \boldsymbol{T}) = 
\begin{cases} 
f\big(g_1^{-1}(\hat{U}), \boldsymbol{T}\big), &\boldsymbol{T} = \boldsymbol{t_1}\\
f\big(g_2^{-1}(\hat{U}), \boldsymbol{T}\big), &\boldsymbol{T} = \boldsymbol{t_0}          
\end{cases}
\end{equation*}
Now consider the following counterfactual query: What would have been the counterfactual value $y'$ of $Y$ had we done $\boldsymbol{t_1}$ given the fact that it was $y$ with $\boldsymbol{t_0}$?
This model will return $f(g_1^{-1}(g_0(U)), \boldsymbol{t_1})$ in response, which can cover a range of different values for arbitrary chosen pairs of $g_1(\cdot)$ and $g_2(\cdot)$.
This concludes the proof by contradiction.
\end{proof}

\subsection{Counterfactual Identifiability in Single Dimensional Bijective Generation Mechanisms}
\label{sub:identifiable}
\begin{theorem}
\label{theo:bijective_identification}
The bijective generation mechanism of $Y$ from its parents $\boldsymbol{T}$ is counterfactually identifiable if $Y$ is a 1-dimensional random variable with a differentiable and strictly increasing \gls{cdf}, and also the ground truth generation mechanism $\big(f(\cdot, \boldsymbol{t})\big)$ is differentiable and strictly increasing (decreasing) for all values of $\boldsymbol{t}$.
\end{theorem}
\begin{proof}
Let $U_1 = f^{-1}_{\theta_1}(Y, \boldsymbol{T})$ and $U_2 = f^{-1}_{\theta_2}(Y, \boldsymbol{T})$. 
Let  $k_1(\cdot) = p_{U_1}(U_1 \le \cdot)$ and $k_2(\cdot) = p_{U_2}(U_2 \le \cdot)$ be the \gls{cdf}s of $U_1$ and $U_2$, respectively.
We use \gls{cdf}s for transforming $U_1$ and $U_2$ to $Unif(0, 1)$.
Let $Z_1 = k_1(U_1)$ and $Z_2 = k_2(U_2)$. 
Due to Probability Integral transform we have $Z_1, Z_2 \sim Unif(0, 1)$. 
Furthermore $Z_2 = s(Z_1, \boldsymbol{T})$ where $s(\cdot, \boldsymbol{T}) = k_2(\cdot) \circ f^{-1}_{\theta_2}(\cdot, \boldsymbol{T}) \circ f_{\theta_1}(\cdot, \boldsymbol{T}) \circ k_1^{-1}(\cdot)$. 
Because $k_2(\cdot)$, $f^{-1}_{\theta_2}(\cdot, \boldsymbol{T})$, $f_{\theta_1}(\cdot, \boldsymbol{T})$, and $k_1^{-1}(\cdot)$ are all diffferentiable and strictly increasing, $s(\cdot, \boldsymbol{T})$ is strictly increasing and differentiable for all realizations of $\boldsymbol{T}$. 
Using the change of variable formula, we have $\forall z_1, \boldsymbol{t}: p_{Z_1|\boldsymbol{T}}(z_1|\boldsymbol{t}) = p_{Z_2|\boldsymbol{T}}(s(z_1, \boldsymbol{t})|\boldsymbol{t}) \cdot |\frac{\partial s(z_1, \boldsymbol{t})}{\partial z_1}|$. 
Since $U_1$ and $\boldsymbol{T}$ are statistically independent (no unobserved confounding assumption), and because $Z_1$ is a deterministic function of $U_1$, we conclude independence of $Z_1$ and $\boldsymbol{T}$. 
By similar argument, $Z_2$ and $\boldsymbol{T}$ are also statistically independent. 
Thus $\forall z_1, \boldsymbol{t}: p_{Z_1}(z_1) = p_{Z_2}(s(z_1, \boldsymbol{t}))  \cdot |\frac{\partial s(z_1, \boldsymbol{t})}{\partial z_1}| \rightarrow \forall z_1, \boldsymbol{t}: \frac{\partial s(z_1, \boldsymbol{t})}{\partial z_1} = 1 \rightarrow \forall \boldsymbol{t}: z_2 = z_1 + H(\boldsymbol{t})$, where $H(\cdot)$ is an arbitrary function. 
$Z_1, Z_2 \sim Unif(0, 1)$ and are statistically independent of $\boldsymbol{T}$, so taking expectation of both sides with respect to the joint distribution of $(Z_1, Z_2)$ results in $\forall \boldsymbol{t}: H(\boldsymbol{t}) = 0$ and hence $Z_1 = Z_2$. 
Now we transform $Z_1$ and $Z_2$ back to $U_1$ and $U_2$ using the inverse \gls{cdf}s: $k_1(U_1) = k_2(U_2)$. 
As  $k_1(\cdot)$ is strictly increasing, it is invertible and $U_1 = k_1^{-1}(k_2(U_2))$. 
This concludes the proof with the nonlinear invertible indeterminacy $g(\cdot)$ being $k_1^{-1}(\cdot) \circ k_2(\cdot)$.
\end{proof}

Note that \Cref{theo:bijective_identification} was first proposed by~\cite{rlcf}, with a different proof.
Also note that counterfactual identifiability of nonlinear additive noise \gls{scm}s~(\cite{deci,elements}), where $Y = f(\boldsymbol{T}) + U$, is implied by \Cref{theo:bijective_identification}.

\subsection{Counterfactual Non-identifiability of Generation Mechanisms with Multi-dimensional Exogenous Variables, An Impossibility Result}
\label{sub:unidentifiability}
Here, we provide an impossibility result, in the form of a general counterexample, for counterfactual identifiability of general generation mechanisms with multi-dimensional exogenous variables, from observational data.

Suppose that we want to learn the generation mechanism for the generation of variable $Y$ in the causal \gls{dag} from its direct causal parents $\boldsymbol{T}$ and corresponding exogenous variable $U$, where $U$ is multi-dimensional.
As described in \Cref{sub:dscm}, similar to prior work, we use conditional generative models to learn this generation mechanism from observational data.
Without loss of generality, suppose $U$ is two dimensional, i.e., $U = (U^1, U^2)$.
Conditional generative models considered here, e.g., \gls{gan}, \gls{nf}, \gls{vae}, etc., sample a latent variable from a tractable distribution, and map it to the observed conditional distribution.
A common choice for the distribution of latent variables is the standard isotropic Gaussian distribution where each dimension of the latent variable is sampled independently from a standard normal distribution.
Even if a different distribution is used for the latent variables, we can still transform this distribution to a standard isotropic Gaussian, and think of the conditional generative transform as the combination of the distribution transform and the original generative transform.
Therefore, without loss of generality, we assume the latent variables are distributed according to a standard isotropic Gaussian.

Let $f$ be the ground truth generation mechanism, where $Y = f(U, \boldsymbol{T})$ and $U$ is sampled from a two-dimensional standard isotropic Gaussian distribution.
We will use $f$ to construct a different generation mechanism $f'$ which fits the observed distribution equally well but provides different counterfactual estimates.

Let us split the domain of $\boldsymbol{T}$ into sub-domains $A$ and $B$.
For $\boldsymbol{T}$ in $A$, we define $f'$ to be the same as the ground truth generation mechanism $f$.
However, for $\boldsymbol{T}$ in $B$, we define $f'$ as first applying an arbitrary rotation to $U$, and then applying the ground truth generation mechanism $f$, i.e.,
\begin{equation*}
f'(U, \boldsymbol{T}) = 
\begin{cases} 
f(U, \boldsymbol{T}), &\forall \boldsymbol{T} \in A\\
f(R(U), \boldsymbol{T}), &\forall \boldsymbol{T} \in B          
\end{cases}
\end{equation*}
where $R(\cdot)$ is an arbitrary rotation.

Since the standard isotropic Gaussian distribution is invariant to rotations, $f$ and $f'$ produce exactly the same conditional distributions, and are valid solutions for the \gls{dscm}'s optimization problem.
However, for counterfactual queries with evidence in $A$ and intervention in $B$ (or the other way around), they generate different counterfactual estimates unless $f$ is rotation invariant.
If $f$ is rotation invariant, it can in fact be parameterized by a single dimension, and two-dimensional $U$ is redundant.
Note that this result holds even if the joint distribution of parents and exogenous variables in training is the same as the joint distribution of counterfactual parents and exogenous variables in counterfactual queries, and we assume access to infinite amount of observational data.
Also note that this result is not limited to bijective generation mechanisms (\cref{sub:bijective}), as we do not assume existence of a deterministic inference mechanism in this section's discussion.

This impossibility result has similarities to, and is partially inspired by the non-identifiability results in unsupervised representation learning~(\cite{locatello2019challenging}) and non-linear \gls{ica}~(\cite{hyvarinen1999nonlinear}), and underscores the need for parametric assumptions about the functional form of generation mechanisms with multi-dimensional exogenous variables for making them counterfactually identifiable.
In the next section, we present a methodology for estimating errors of counterfactual estimates given a particular set of assumptions.
\section{A practical approach}
\label{sec:method}

As we saw in \Cref{sec:motiv,sub:unidentifiability}, even in presence of an infinite amount of data and assuming perfect optimization, learning from observational datasets may result in faulty \gls{dscm}s for the task of counterfactual inference.
In \Cref{sub:identifiable}, we characterized a small set of bijective generation mechanisms for which we have counterfactual identifiability with guarantees.
In \Cref{sub:unidentifiability}, we proved an impossibility result for the counterfactual identifiability of general generation mechanisms with multi dimensional exogenous variables. 
However, adding more inductive biases in terms of parametric assumptions about the true \gls{scm} may result in counterfactual identifiability.

Proving counterfactual identifiability for each set of assumptions is a cumbersome process.
Also, exact counterfactual identifiability is often too strong, e.g., in cases where low counterfactual error is tolerable by practitioners.
To overcome these challenges, we propose a computational analysis of counterfactual identifiability with \gls{dscm}s.
Given a set of parametric and functional form assumptions, our method can calculate different error metrics for counterfactual estimates.

\subsection{Method}

\subsubsection{The First Step}
As explained in \Cref{sec:background} and similar to \cite{Pawlowski2020DeepSC}, for each variable $Y$ in the causal \gls{dag}, we train a conditional generative model and a conditional inference mechanism with conditions being its causal parents $\boldsymbol{T}$.
In this step, we can add all the inductive biases in the architecture of the conditional generative model, e.g., monotonicity~(\cite{monotone1,monotone2,i-resnet}), specific functional forms and parametric assumptions, smoothness~(\cite{spec-norm,gradient-penalty,wgan}), etc.
The loss function for training the conditional generative model depends on the type of generative model, e.g., \gls{elbo} for \gls{vae}s, discrimination loss for \gls{gan}s, log-likelihood for \gls{nf}s, etc.
We call this loss function $L_{\text{generation}}$.
Intuitively, these models and objective functions learn the conditional distribution of $Y$ given each realization of its parents $\boldsymbol{t}$, i.e., $p_{Y|\boldsymbol{T}}(\cdot|\boldsymbol{T})$.
\begin{equation*}
    \theta_1 = \argmin_{\theta} (L_{\text{generation}}(\theta))
\end{equation*}
The model resulting from this step is a typical instance that a \gls{dscm} training would produce.

\subsubsection{The Second Step}
For each endogenous variable in the causal DAG, we train a second conditional generative model and inference mechanism, like the previous step, but with a different loss function.
This time, we attempt to fit the observational conditional distribution $p_{Y|\boldsymbol{T}}(\cdot|\boldsymbol{T})$ as well as the first model, but also to maximize disagreement with the counterfactual estimates of the first model.
We can measure this disagreement in terms of any error metric that we care about, e.g., KL-divergence.
\begin{equation*}
    L_{\text{disagreement}} = -E_{\boldsymbol{T_1}, \boldsymbol{T_2}, Y_1 \sim \mathbb{S}} \Big[D\big(p(Y_{\boldsymbol{T_2}}|\boldsymbol{T_1}, Y_1;\theta_1), p(Y_{\boldsymbol{T_2}}|\boldsymbol{T_1}, Y_1;\theta_2)\big)\Big]
\end{equation*}
$\mathbb{S}$ can be any desired distribution.
In our experiments, we choose the intervention and evidence independently from the empirical distribution of observations.

The second model, when trained to convergence, will be an example of a worst-case possible ground truth \gls{scm} in terms of counterfactual accuracy metric (disagreement loss) for the first model, which respects all the parametric assumptions.
\begin{align}
\label{eqn:constrained}
&\theta_2 = \argmin_{\theta} (L_{\text{disagreement}}(\theta))\\ 
\textrm{subject to} \quad   &L_{\text{generation}}(\theta) = L_{\text{generation}}(\theta_1) \nonumber
\end{align}
If the worst-case counterfactual error obtained in step two is negligible enough for the practitioner, they can move on with their modeling with the trained \gls{dscm}.
On the other hand, if the counterfactual estimation error is not sufficiently small, it signals the need for more inductive biases, parametric assumptions, etc. for this endogenous variable.

To make \Cref{eqn:constrained} amenable to gradient based optimization, we use the Lagrange multiplier to construct $L_{\text{second}}$:
\begin{equation}
\label{eqn:lagr}
    L_{\text{second}}(\theta) = L_{\text{disagreement}}(\theta) + \lambda L_{\text{generation}}(\theta)
\end{equation}
Note that we removed $\lambda L_{\text{generation}}(\theta_1)$ as it is a constant which does not depend on $\theta$.  
\section{Evaluation}

In this section, we evaluate the practical method presented in \Cref{sec:method} in an identifiable case directly adopted from \cite{Pawlowski2020DeepSC}, and a non-identifiable case that we create ourselves.

\subsection{An Identifiable Case}
\label{sub:eval_identifiable}

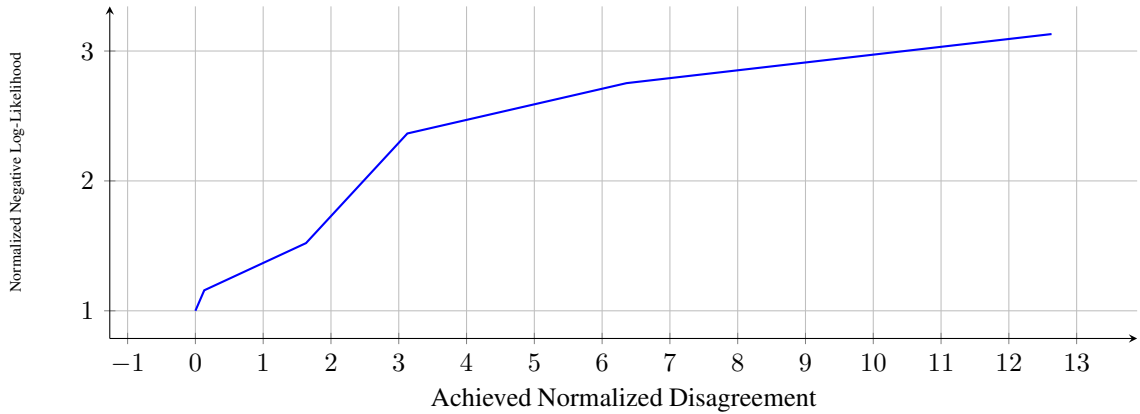
\begin{figure}[tb]
    \centering
    \begin{tikzpicture}
        \begin{axis}[height=6cm,
                     width=\linewidth,
                     axis lines=left,
                     xlabel={Achieved Normalized Disagreement},
                     ylabel={\tiny Normalized Negative Log-Likelihood},
                     font=\small,
                     enlargelimits=true,
                     xmin=0,
                     grid=both
                     ]
            \addplot[solid, blue, thick, mark=none] coordinates{
            (0.0, 1) (0.1307, 1.1579092539) (1.632, 1.5211694939) (3.127, 2.3651910081) (6.36, 2.7527237622) (12.63, 3.1306026755)
            };
        \end{axis}
    \end{tikzpicture}
  \caption{Normalized negative log-likelihood increases consistently with the increase in the achieved normalized disagreement.}
    \label{fig:result:flow}
\end{figure}

In this part, we use ``case study 1'' in \cite{Pawlowski2020DeepSC} as a counterfactually identifiable example to evaluate the practical method presented in \Cref{sec:method}.
In this example, intensity ($Y$) and thickness ($T$) are generated according to the following \gls{scm}:
\begin{align}
    T &= 0.5 + U_T \,, & U_T &\sim \Gamma(10, 5) \, \nonumber\\ 
    Y &= f(U_Y, T) = 191 \cdot \sigma(0.5 \cdot U_Y + 2 \cdot T - 5) + 64 \,, & U_Y &\sim \mathcal{N}(0, 1)
    \label{eqn:intensity}
\end{align}
where $\sigma(\cdot)$ is the logistic sigmoid.

We focus on learning the generation mechanism for intensity ($Y$).
\Cref{eqn:intensity} is a bijective single dimensional generation mechanism that satisfies all assumptions of \Cref{theo:bijective_identification} and is therefore counterfactually identifiable.
It was shown in \cite{Pawlowski2020DeepSC} that using a conditional \gls{nf} can learn this generation mechanism, and estimate counterfactuals with high accuracy.
For the first step of our method, we exactly follow the methodology and setup of \cite{Pawlowski2020DeepSC}:
We assume that we know the bijective functional form of the generation mechanism, i.e., the mapping from $U_Y$ to $Y$ is a combination of a conditional affine transform, sigmoid transform, and finally a normalization transform.
Only the parameters of the conditional affine transform depend on the value of $T$.
We use a learnable linear transformation to learn the location and scale of this conditional affine transform based on values of $T$.
It is worth emphasizing that $L_{\text{generation}}$ is equal to negative log-likelihood of the samples when using a \gls{nf}.

For the second step, we use the same setup for learning the \gls{nf}, but we change the loss function to \Cref{eqn:lagr}.
For the disagreement loss, we use the \gls{mse} of the estimated counterfactuals compared with the conuterfactuals estimated by the first model.
We normalize this \gls{mse} with the \gls{mse} of predicted counterfactuals from the first model compared with the evidence in the counterfactual query, when the intervention and evidence are both selected randomly from the observed thicknesses.

We found \Cref{eqn:lagr} to be an unstable loss function as $L_{\text{disagreement}}$ can grow unbounded.
As a result, we optimize the following loss function while sweeping the range of ``\emph{disagreement threshold}'' as a hyper-parameter:
\begin{equation}
\label{eqn:threshold}
    L_{\text{second}}^* = L_{\text{generation}} + (L_{\text{disagreement}}-\textrm{disagreement threshold})^2
\end{equation}

As we see in \Cref{fig:result:flow}, the negative log-likelihood achieved by the second \gls{nf} starts increasing immediately as we encourage disagreement.
This means that in this case, no \gls{nf} exists that achieves the same log-likelihood as the first \gls{nf}, but gives different counterfactual estimates. 
In other words, the practical method confirms counterfactual identifiability of \gls{dscm} using optimization, without going through identifiability proofs.

\subsection{A Non-identifiable Case}

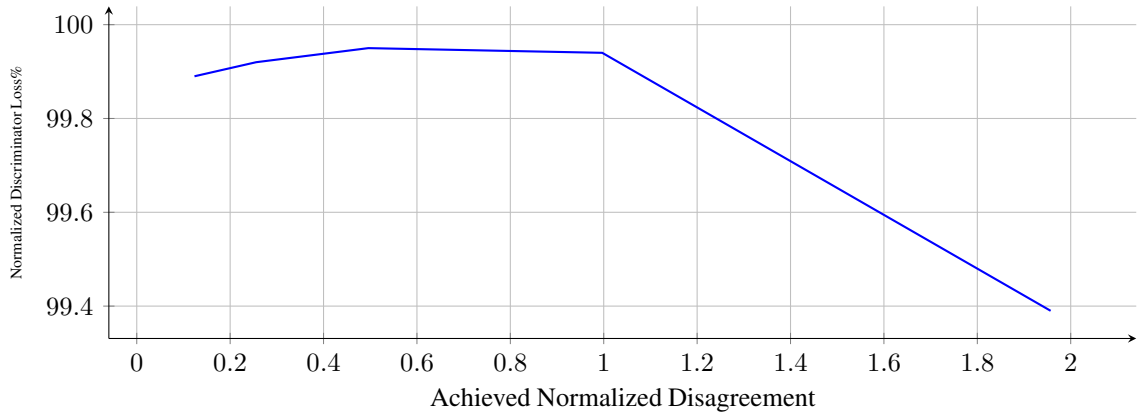
\begin{figure}[tb]
    \centering
    \begin{tikzpicture}
        \begin{axis}[height=6cm,
                     width=\linewidth,
                     axis lines=left,
                     xlabel={Achieved Normalized Disagreement},
                     ylabel={\tiny Normalized Discriminator Loss\%},
                     font=\small,
                     enlargelimits=true,
                     ymax=99.98, 
                     grid=both
                     ]
            \addplot[solid, blue, thick, mark=none] coordinates{
            (0.1236, 99.89) (0.2558, 99.92) (0.496, 99.95) (0.9968, 99.94) (1.957, 99.39)
            };
        \end{axis}
    \end{tikzpicture}
  \caption{Normalized discriminator loss is high until a certain achieved normalized disagreement (around 1), which is the counterfactual estimation worst-case error.}
    \label{fig:result:discriminator}
\end{figure}

To evaluate our method in non-identifiable case, we modify \Cref{eqn:intensity} to create a new \gls{scm}.
We change $U_Y$ to a two-dimensional random variable where each dimension is sampled independently from $Unif(-\sqrt3, \sqrt3)$, and $Y$ to a two-dimensional random variable.
To set the generation mechanism of the two-dimensional $Y$, we randomly initialize a fully connected neural network $f_{\theta^*}$ that takes as input $(U_Y, T)$ and gives as output the two-dimensional corresponding value of $Y$.
We finetune its weights in a way that the \gls{mse} of counterfactual estimates compared with the evidence, $E_{T, T', U_Y \sim Unif(-\sqrt3, \sqrt3)}\big[\big(f_{\theta^*}(U_Y, T) - f_{\theta^*}(U_Y, T')\big)^2\big]$, is close to $1$ (to normalize the disagreement loss in the second step).
We also train a separate deterministic inference network $(f^{-1}_{\theta^*}(\cdot, T))$.
This enforces $f_{\theta^*}$ to be a bijection.

Now we apply our method to this \gls{scm}.
For the first step, we train a conditional \gls{gan} to fit the conditional distribution $p_{Y|T}(\cdot|T)$ generated by $f_{\theta^*}$.
We assume that by prior knowledge, we know that the generation mechanism of $Y$ is a bijection.
To inject this inductive bias into the training process of the conditional \gls{gan}, we train a separate deterministic inference mechanism along with the conditional \gls{gan}.
Note that $L_{\text{generation}}$ is equal to the negataive discrimination loss in this case.

For the second step of the method which trains a second conditional GAN that disagrees with the first one in counterfactual estimations. We define the disagreement loss as the \gls{mse} of counterfactual predictions of the first and the second model.
We find \Cref{eqn:lagr} to be an unstable loss function as $L_{\text{disagreement}}$ can grow unbounded.
Therefore, similar to what we did in \Cref{sub:eval_identifiable}, we use \Cref{eqn:threshold} as the loss function with a few different thresholds.

\Cref{fig:result:discriminator} shows the discriminator loss (normalized by its theoretical upper bound), which is an indicator of how well \gls{dscm} captures the conditional distribution $p_{Y|T}(\cdot|T)$, vs the achieved normalized disagreement.
As you can see from the trend, distribution quality keeps being high until a certain point (around normalized threshold of 1) where it starts to fall.
This means that this \gls{scm} is not counterfactually identifiable, as we expected from \Cref{sub:unidentifiability}.
Additionally, this trend also demonstrates that the average (over evidence and counterfactual parents) worst case (over ground truth \gls{scm}s) normalized counterfactual error is $\sim 1$.

Practitioners can decide whether this counterfactual error is sufficiently small for their use case.
If yes, they can use the \gls{dscm} for counterfactual inference.
If not, they should add more inductive biases such as functional forms, parametric assumptions, etc. until the counterfactual error meets their criterion.
\section{Concluding Remarks}

A recent line of work~(\cite{Dash2022EvaluatingAM,Khemakhem2021CausalAF,Pawlowski2020DeepSC,Sanchez2022DiffusionCM}) suggests use of conditional generative models for learning \gls{dscm}s, which can further be utilized for the task of counterfactual inference.
In this work, we demonstrated how this approach can fail in the absence of counterfactual identifiability.
We characterized the set of bijective single dimensional generation mechanism as counterfactually identifiable, and also proved an impossibility result for counterfactual identifiability of general generation mechanisms with multi-dimensional exogenous variables.
Furthermore, we proposed a practical algorithm that can help a practitioner decide whether or not they want to use deep models in their specific setting with their set of assumptions, by giving informative bounds on counterfactual estimate errors of \gls{dscm}s.
We evaluated this method in an identifiable and a non-identifiable setting.

\newpage


\bibliography{clear2023}

\end{document}